\documentclass[letterpaper, 10 pt, conference]{ieeeconf}  % Comment this line out
\IEEEoverridecommandlockouts                              % This command is only
\usepackage{balance} % for balancing columns on the final page
\usepackage{graphics} % for pdf, bitmapped graphics files
\usepackage{epsfig} % for postscript graphics files
\usepackage{amsmath} % assumes amsmath package installed
\usepackage{amssymb}  % assumes amsmath package installed
\usepackage{amsfonts}
\usepackage{xcolor}
\usepackage{bm}
\usepackage{algorithm}
\usepackage[noend]{algorithmic}
\usepackage{url}
\usepackage{gensymb}
\usepackage{booktabs}
\usepackage{adjustbox}

\newcommand{\BlackBox}{\rule{1.5ex}{1.5ex}}  % end of proof
\ifdefined\proof
    \renewenvironment{proof}{\par\noindent{\bf Proof\ }}{\hfill\BlackBox\\[2mm]}
\else
    \newenvironment{proof}{\par\noindent{\bf Proof\ }}{\hfill\BlackBox\\[2mm]}
\fi
 
\newtheorem{theorem}{Theorem}
\newtheorem{lemma}{Lemma}

\newtheorem{definition}{Definition}

\newcommand{\T}[1]{\bar{\mathcal{T}}_{#1}^{*}}
\newcommand{\bH}[1]{\bar{H}_{#1}^{*}}
\DeclareMathOperator*{\argmax}{arg\,max}

\newtheorem{assumption}{Assumption}

\newcommand{\revision}[1]{\textcolor{black}{#1}}

\title{\LARGE \bf
Non-Stationary Policy Learning for Multi-Timescale Multi-Agent Reinforcement Learning
}

\author{Patrick Emami, Xiangyu Zhang, David Biagioni and Ahmed S. Zamzam% <-this % stops a space
\thanks{This work was authored in part by the National Renewable Energy Laboratory (NREL), operated by Alliance for Sustainable Energy, LLC, for the U.S. Department of Energy (DOE) under Contract No. DE-AC36-08GO28308. This work was supported by the Laboratory Directed Research and Development (LDRD) Program at NREL. The views expressed in the article do not necessarily represent the views of the DOE or the U.S. Government. The U.S. Government retains and the publisher, by accepting the article for publication, acknowledges that the U.S. Government retains a nonexclusive, paid-up, irrevocable, worldwide license to publish or reproduce the published form of this work, or allow others to do so, for U.S. Government purposes. A portion of this research was performed using computational resources sponsored by the Department of Energy's Office of Energy Efficiency and Renewable Energy and located at the National Renewable Energy Laboratory.}
\thanks{Patrick Emami, Xiangyu Zhang, and Ahmed Zamzam are with the National
Renewable Energy Laboratory, Golden, CO 80401 USA (e-mail:
{\{{patrick.emami, xiangyu.zhang, ahmed.zamzam\}@nrel.gov}}.}%
\thanks{David Biagioni is with Maplewell Energy,
        Broomfield, CO 80021, USA (email: {{{dave@maplewelleng.com)}}}.}%
}

\begin{document}

\maketitle
\thispagestyle{empty}
\pagestyle{empty}

\begin{abstract}
In multi-timescale multi-agent reinforcement learning (MARL), agents interact across different timescales.
In general, policies for time-dependent behaviors, such as those induced by multiple timescales, are non-stationary.
Learning non-stationary policies is challenging and typically requires sophisticated or inefficient algorithms.
Motivated by the prevalence of this control problem in real-world complex systems, we introduce a simple framework for learning non-stationary policies for multi-timescale MARL. 
Our approach uses available information about agent timescales to define a periodic time encoding.
\revision{In detail, we theoretically demonstrate that the effects of non-stationarity introduced by multiple timescales can be learned by a periodic multi-agent policy. 
To learn such policies, we propose a policy gradient algorithm that parameterizes the actor and critic with phase-functioned neural networks, which provide an inductive bias for periodicity.}
The framework's ability to effectively learn multi-timescale policies is validated on a gridworld and building energy management environment.
\end{abstract}

%%%%%%%%%%%%%%%%%%%%%%%%%%%%%%%%%%%%%%%%%%%%%%%%%%%%%%%%%%%%%%%%%%%%%%%%%%%%%%%%
%
\section{INTRODUCTION}
The ability to control multiple interacting components is essential to efficiently manage complex systems. For instance, in power systems applications, flexible loads and distributed generation operate within this complex system with coupling in the dynamical models, constraints, and objectives. Similar multi-component control challenges appear in transportation systems, robotics, etc.
Thus, interest in data-driven approaches that try to learn distributed control policies from experience has grown recently.
Multi-agent reinforcement learning (MARL) is a promising agent-based sequential decision-making framework for learning complex coordination strategies.
Crucially, it does not depend on having access to a model of the (often stochastic and nonlinear) environment dynamics.

However, applying MARL to real-world problems is challenging because agents typically only receive noisy, partial observations of the environment state and have limited communication with each other.
Moreover, from the perspective of each agent, the environment dynamics appear to shift over time as other agents learn to adapt their behaviors.
All of these factors contribute to the extreme \emph{non-stationarity} present in MARL, which makes learning good agent policies notoriously challenging~\cite{papoudakis2019dealing,HernandezLeal2019}. 
Despite the dedication of significant effort to discover practical MARL algorithms capable of learning policies in the face of non-stationarity~\cite{Lowe2020,Foerster2017}, there remains a gap between the synthetic environments used for algorithm development and the real world.

This work studies an under-explored MARL setting inspired by real world applications where agents need to coordinate \emph{time-dependent} actions across different timescales.
This type of time-dependent coordination, arises, for example, in:
\begin{itemize}
    \item[\textbf{Ex 1}:] Power systems applications where we wish to learn a coordination strategy between electrical devices that can be controlled at different timescales (e.g., energy storage units, solar panels, and thermostatically controlled loads) connected to the same micro-grid (Fig.~\ref{fig:demo}),
    \item[\textbf{Ex 2}:] Robotic control tasks where multiple heterogeneous robots try to collaborate to execute a task (e.g., moving an object) while their actuators are controlled at different frequencies.
\end{itemize} 
Introducing a time dependency via multiple timescales to MARL adds \emph{additional} sources of non-stationarity, which makes learning policies particularly challenging. 

\begin{figure}[t]
    \centering
    \includegraphics[width=\columnwidth]{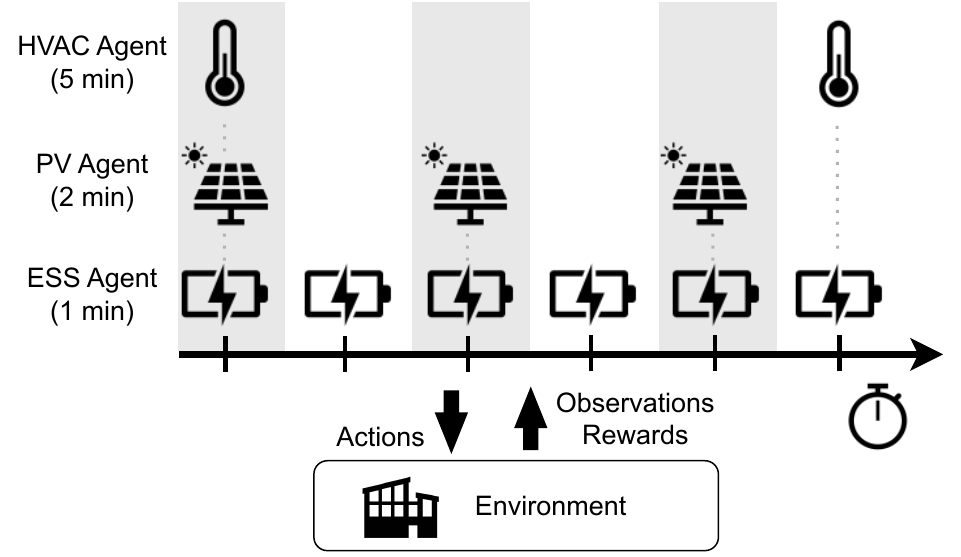}
    \caption{In \textbf{Multi-timescale MARL}, agents that act on different timescales try to learn to coordinate to achieve a goal, e.g., controlling a heterogeneous set of electronic devices for building energy optimization. Learning a \emph{non-stationary} multi-agent policy allows multi-timescale agents to perform complex time-dependent behaviors.}
    \label{fig:demo}
    \vspace{-1.5em}
\end{figure}
In this work, we \revision{formally define the multi-timescale decentralized partially observable Markov decision process (Dec-POMDP) setting} and propose a framework for learning multi-agent time-dependent (i.e., non-stationary) policies \revision{to solve such environments.
We show that multiple timescales induces the optimal multi-timescale policy to be periodic in nature.
A practical policy gradient method for learning periodic multi-agent policies based on phase-functioned neural networks~\cite{Holden2017} is provided.
Our framework's ability to learn effective multi-timescale policies with fewer environment interactions than key baselines is validated on a gridworld and a building energy management environment.}
\section{\revision{Multi-timescale Dec-POMDPs}}
\label{sec:bg}
First, we \revision{define a Dec-POMDP~\cite{bernstein2002complexity}, which} is a multi-agent stochastic game defined by a tuple: 
\begin{equation}
\label{eq:dec-pomdp}
\texttt{DEC-POMDP} := (N, S, \mathbf{A}, \mathbf{O}, r, T, \revision{dt}, \gamma, \rho, p, P).
\end{equation}
The number of agents is $N$, $S$ is the state space of the environment, $\mathbf{A} := \times_{i=1}^N A^i$ is a joint action space, $\mathbf{O} := \times_{i=1}^N O^i$ is a joint observation space, $r(s,a) : S \times \mathbf{A} \rightarrow \mathbb{R}$ is a global reward function, $T$ is the (possibly infinite) horizon, \revision{$dt$ is the time discretization of the environment}, $\gamma \in [0,1]$ is a discount factor, and $\rho : S \rightarrow [0,1]$ is the initial state distribution.
The joint action $\mathbf{a} \in \mathbf{A}$ induces a transition from state $s$ to state $s'$ according to the transition function $p(s, \mathbf{a}) : S \times \mathbf{A} \rightarrow S$.
We assume that each agent is only given a noisy, partial observation of the state governed by an observation function $P(s,i) : S \times \mathbb{N} \rightarrow O^i$, where $i$ is the agent index, and that there is limited or no communication between agents.
Each agent learns a stochastic policy function $\pi^i(a^i \mid o^i)$ that maps an observation to a distribution over $A^i$, and the joint policy is $\bm{\pi} = \{\pi^i,\dots,\pi^N\}$.

\revision{In this work, we introduce multi-timescale Dec-POMDPs as an extension of Dec-POMDPs:
\begin{equation}
\label{eq:mt-dec-pomdp}
\texttt{MT-DEC-POMDP} := [\texttt{DEC-POMDP}; k, C],
\end{equation}
where $[;]$ is concatenation. The agents are defined with action frequencies $k := \{k_1,\dots,k_N\}$ and the environment is defined to be periodic with period $C \geq 1$, i.e., the reward function $r_t(s, \mathbf{a})$ and transition functions $p_t( s, \mathbf{a} )$ are $C$-periodic. 
When $C = 1$, the environment is aperiodic, as is typically assumed in a \texttt{DEC-POMDP}.
Each agent's timescale is defined by the action frequency $k_i$ times the base timescale discretization $dt$.
For example, agent 1 with $k_1 = 2$ acts every \revision{two steps} (i.e., when $t \bmod 2 dt = 0$) and agent 2 with $k_2 = 3$ acts every \revision{three steps} (i.e., when $t \bmod 3 dt = 0$).
We assume environments are defined with $dt = 1$ for simplicity in the remainder of the paper.
Between actions, ``slow'' agents take a null action $a^i_{\texttt{null}}$ or repeat the most recently taken action (e.g., when an action represents the setpoint for a device).}
Our setting differs from the multi-timescale MARL setting considered in Wu et al.~\cite{Wu2020} as they assume agents can communicate locally with their neighbors and that the environment is aperiodic.

\revision{One can observe that the sequencing of actions across agent timescales repeats periodically every $\Tilde{K} = \text{LCM}(k_1,\dots,k_N)$ time steps, where $\text{LCM}$ is the least common multiple.
Taking into account the periodicity of the environment $C$, we see that the pattern of action and state transition sequencing repeats every $K = \text{LCM}(\Tilde{K},C)$ steps. 
When $K > 1$, solving a \texttt{MT-DEC-POMDP} is generally more difficult than solving a \texttt{DEC-POMDP}, as the periodicity introduces a time dependency that compounds on top of the non-stationarity caused by partial observability and limited communication between agents~\cite{Oliehoek2008,Lowe2020}}.

To help motivate our framework, we now briefly describe a class of MARL environments where the optimal \revision{stationary} joint policy is sub-optimal (i.e., when the agents ignore time dependencies).
In detail, \revision{certain MARL environments cause agents to suffer} from what we call the \emph{observation aliasing problem}.
\revision{In Sec.~\ref{sec:movebox}, we provide an example multi-timescale environment exhibiting observation aliasing.}
This aliasing problem, which is closely related to the state aliasing problem that can arise in single agent RL~\cite{pardo2018time}, occurs when agent $i$ receives a specific partial observation $o^i$ that has a different action under \revision{non-stationary} policy $\pi^i$ depending on the time step:
\begin{equation*}
    \label{eq:alias}
    \exists t,t' \text{ with } t<t' \text{ s.t. } \pi_t^i(o^i_t) \neq \pi_{t'}^i( o^i_{t'} ).
\end{equation*}
A \emph{time-unaware} agent is not able to learn a stationary policy that distinguishes between $o^i_t$ and $o^i_{t'}$ due to the impossibility of proper credit assignment.
Intuitively, time-unaware agents perceive time-dependent rewards or dynamics as stochasticity in the environment (e.g., from an exogenous source).
Consequentially, these agents learn suboptimal policies that try to explain the aliased observation as perceived randomness.

A common heuristic is to make each agent time aware by appending the current time step $t$ to the observation.
Alternatively, recurrent neural networks can be used to infer a \emph{belief state} for each agent based on its history~\cite{wang2020r}.
In this work, we explore how information about agent timescales and environment periodicity can be used to more effectively learn non-stationary policies in \texttt{MT-DEC-POMDP}s.

Recent work has introduced custom multi-timescale solutions for power systems problems that simply \revision{attempt to learn a stationary policy~\cite{wu2023multi}} or use recurrent neural networks to model temporal information for a fast agent and a slow agent~\cite{Ochoa2022}.
The single agent RL setting with multiple action frequencies is also a closely related setting.
Multi-timescale MDPs (MMDPs)~\cite{HSC2003,He2011} involve an agent formulated as a hierarchical MDP with fast and slow actions.
MMDPs aim to learn a (top-down) hierarchy over action timescales, in the sense that actions taken on slower timescales influence the actions taken on faster timescales, but not vice versa.
Also related is a setting where a factored-action MDP agent can choose actions that persist for various lengths of time~\cite{lee2020reinforcement}.

\section{$K$-periodic Non-stationary Joint Policies}
\label{sec:mt}
We now demonstrate that, under simplifying assumptions, policy iteration in the space of \emph{$K$-periodic} non-stationary joint policies converges to the optimal multi-timescale policy and value function. 
\begin{definition}
For any $K\geq 1$, a $K$-periodic non-stationary joint policy satisfies
\begin{equation}
    \bm{\bar{\pi}}(\mathbf{o}_t) = \bm{\bar{\pi}}(\mathbf{o}_{t+{K}}).
\end{equation}
\end{definition}
Let $\bm{\bar{\Pi}}_{K} = \big\{ \bm{\bar{\pi}}: \bm{\bar{\pi}}(\mathbf{o}_t) = \bm{\bar{\pi}}(\mathbf{o}_{t+K})\big\}$ be the set of all such policies.
We use this to define a \emph{multi-timescale non-stationary joint policy} $\bm{\pi}_t \in \bm{\Pi}$ with timescale action frequencies $k$ as the policy \emph{induced} by a $K$-periodic non-stationary policy:
\begin{align}
   &\forall t, \bm{\pi}_t = \{\pi_t^1,\dots,\pi_t^N\}\text{ s.t. } \\
   &\pi_t^i := \begin{cases}\bar{\pi}_t^i \hspace{.5em} &\text{if } t  \bmod k_i = 0\\ \delta_{t-(t \bmod k_i)}(a^i) \text{ or } \delta_{t}(a^i_{\texttt{null}})\hspace{.5em} &\text{otherwise.} \end{cases}\nonumber
\end{align}
Here, $\delta_t$ is the Dirac delta function. In general, the inducing policy for a multi-timescale policy $\bm{\pi}_t$ does not need to be $K$-periodic. However, later in this section we prove that policy iteration in the space of $K$-periodic policies converges to the optimal multi-timescale policy assuming cooperative agents and full observability.
\begin{definition}
The projection operator for joint actions $\Gamma_t^k (\mathbf{a}_t)$ is defined as 
\begin{align}
&\Gamma_t^k(\mathbf{a}_t) = \{\bar{a}^1_t,\dots, \bar{a}^N_t\} \text{, where }\\
&\bar{a}^i_t := 
\begin{cases} a^i_t \hspace{.5em} &\text{if } t \bmod k_i = 0 \\ \bar{a}^i_{t-1} \text{ or } a^i_{\texttt{null}} \hspace{.5em} &\text{otherwise.} \end{cases}
\end{align}
\end{definition}
Notice that $\Gamma_t^k$ is $\Tilde{K}$-periodic. 
That is, $\Gamma_t^k(\mathbf{a}) = \Gamma_{t+ \Tilde{K}}^k(\mathbf{a})$ for any $t$ and joint action $\mathbf{a}$.
\begin{definition}
The projection operator for the reward and transition function is defined as
\begin{equation}
    \Theta_t^C ( f_t ) = \begin{cases} f_t^0 \hspace{.5em}&\text{if } t \bmod C = 0\\ \vdots\\
    f_t^{C-1} \hspace{.5em}& \text{if } t \bmod C = C-1.\end{cases}
\end{equation}
\end{definition}
Notice that $\Theta_t^C$ is $C$-periodic, i.e., $\Theta_t^C(r(s,\mathbf{a})) = \Theta_{t+C}^C(r(s,\mathbf{a}))$ for any $t$ and $(s,\mathbf{a})$.

The agents aim to learn an optimal multi-timescale policy \begin{equation}
    \bm{\pi}^* = \argmax_{\bm{\pi} \in \Pi} \mathbb{E}_{\bm{\pi},\Theta_t^C(p_t(s,a))} \Biggl[ \sum_{t=0}^T \gamma^t \Theta_t^C \bigl( r_t(s_t, \Gamma_t^k(\mathbf{a}_t)) \bigr) \Biggr],
\end{equation}
and the optimal multi-timescale value function is 
\begin{equation}
Q_t^{\bm{\pi}^*}(a, s) = \mathbb{E}_{\mathbf{a_{t+1}} : \mathbf{a_T}, s_{t+1}:s_T} \bigl[ R_t | \Gamma_t^k(\mathbf{a_t}), s_t, \bm{\pi}^* \bigr],
\end{equation}
where $R_t = \sum_{\tau=t}^T \gamma^{\tau - t} \Theta^C_{\tau} r_{\tau}(s_{\tau}, \mathbf{a}_{\tau})$.
The following theorem establishes convergence to the optimal multi-timescale value function via \emph{policy iteration in the space of $K$-periodic non-stationary policies} assuming agent cooperation and full observability.

\begin{assumption}
\label{assumption}
The agents are cooperative and have full observability of the environment.
\end{assumption}

\begin{theorem}
\label{thm:main}
Under assumption 1, for any $\bm{\pi}^0 \in \Pi$ induced by $K$-periodic non-stationary policy $\bar{\bm{\pi}}^0 \in \bar{\Pi}_K$, $K = \text{LCM}(\Tilde{K},C)$, the sequence of value functions $Q^{\bm{\pi}^n}$ and improved policies $\bm{\pi}^{n+1}$ due to policy iteration converges to the optimal multi-timescale value function and optimal multi-timescale policy $\bm{\pi}^*$, i.e., $Q^{\bm{\pi}^*}(s,a) = \lim_{n \rightarrow \infty} Q^{\bm{\pi}^n}(s,a) \geq Q^{\bm{\pi}}(s,a)$. Moreover, the optimal multi-timescale value function always exists and is induced by a $K$-periodic non-stationary policy.
\end{theorem}

\begin{lemma}
\label{lemma:eq}
Under assumption~\ref{assumption}, \texttt{MT-DEC-POMDP} reduces to a multi-timescale multi-agent MDP~\cite{boutilier1996planning}. Furthermore, a multi-timescale multi-agent MDP is equivalent to an action-persistent factored action (FA) MDP~\cite{lee2020reinforcement}.
\end{lemma}
\begin{proof}
A multi-agent MDP can be described by the tuple $\{S, N, \mathbf{A}, p, r \}$ with elements defined as in Sec.~\ref{sec:bg}~\cite{boutilier1996planning}.
A multi-timescale multi-agent MDP is defined similarly to \texttt{MT-DEC-POMDP} (Eq.~\eqref{eq:mt-dec-pomdp}), i.e., by expanding a multi-agent MDP with $k$ and $C$.
Assuming agent cooperation, \texttt{MT-DEC-POMDP} has a single reward shared by all agents.
Assuming full observability, the observation function $P(s,i)$ in a \texttt{MT-DEC-POMDP} is the identity mapping.
Therefore, each agent's policy in \texttt{MT-DEC-POMDP} becomes $\pi^i(a^i | s)$, which completes the reduction. 
A factored action (FA) MDP with a fully factorized policy over $N$-dimensional actions $\pi(a^1, \dots, a^N | s) = \prod_{n=1}^N \pi^n (a^n | s)$ is a single agent MDP that can be thought of as equivalently having $N$ agents, i.e., a multi-agent MDP.
A $k$-persistent FA MDP assumes that action $a^i$ is decided every $k_i$ steps and repeated otherwise~\cite{lee2020reinforcement}, and can similarly be extended to periodic environments with period $C$.
By modifying the persistence property to allow action $a^i$ to be repeated $k_i$ times \emph{or} for a null action $a^i_{\texttt{null}}$ to be subsequently taken $k_i - 1$ times, the $k$-persistent FA MDP is equivalent to the multi-timescale multi-agent MDP setting, completing the proof. 
\end{proof}

Given Lemma~\ref{lemma:eq}, our proof for Theorem~\ref{thm:main} follows the same proof technique used to prove Theorem 3 in Lee et al.~\cite{lee2020reinforcement}.
Differently, our proof handles $C$-periodic environments. 
To that end, we define here the one-step multi-timescale Bellman \emph{optimality} operator $\bar{\mathcal{T}}_t^{*}$ induced by $\bm{\bar{\pi}}$ for $t \in \{0, \dots, K - 1\}$:
\begin{align}
    &(\bar{\mathcal{T}}_t^{*} Q)(s, \mathbf{a}) := \\
    &\Theta_t^C(r_t(s_t, \mathbf{a}_t)) + \gamma \mathbb{E}_{s_{t+1}\sim \Theta_t^C( p_t ) }\bigl[ \max_{\mathbf{a}_{t+1}} Q(s_{t+1}, \Gamma_t^k(\mathbf{a}_{t+1}) \bigr].\nonumber
\end{align}
Notice that $\bar{\mathcal{T}}_t^{*}$ is $K$-periodic due to the $\Tilde{K}$-periodic action projection $\Gamma_t^k$ and the $C$-periodic projection operator $\Theta_t^C$. 
Thus, $\bar{\mathcal{T}}_t^{*} Q = \bar{\mathcal{T}}_{t+K}^{*} Q$ for any $t$ and $Q$. Now, we define the $K$-step multi-timescale Bellman optimality operator $\bar{H}_t^*$ by composing one-step Bellman optimality operators as follows:
\begin{align}
    (\bar{H}_0^* Q) (s,a) &:= (\bar{\mathcal{T}}_0^{*}\bar{\mathcal{T}}_1^{*} \cdots \bar{\mathcal{T}}_{K-2}^{*} \bar{\mathcal{T}}_{K-1}^{*} Q)(s,a)\\
    (\bar{H}_1^* Q) (s,a) &:= (\bar{\mathcal{T}}_1^{*}\bar{\mathcal{T}}_2^{*} \cdots \bar{\mathcal{T}}_{K-1}^{*} \bar{\mathcal{T}}_{K}^{*} Q)(s,a) \nonumber\\
    &\phantom{:}\vdots \nonumber\\ 
    (\bar{H}_{K-1}^* Q) (s,a) &:= (\bar{\mathcal{T}}_{K-1}^{*}\bar{\mathcal{T}}_K^{*} \cdots \bar{\mathcal{T}}_{K-3}^{*} \bar{\mathcal{T}}_{K-2}^{*} Q)(s,a).\nonumber
\end{align}
The next lemma establishes that $K$-step multi-timescale Bellman optimality operators are a contraction mapping.
\begin{lemma}
For all $t \in \{ 0, \dots, K-1\}$, the $K$-step multi-timescale Bellman optimality operator $\bar{H}_t^*$ is a $\gamma^K$-contraction with respect to infinity norm with $\bar{H}_t^* Q_t^* = Q_t^*$ as the unique fixed point solution. That is, for any $Q_t^0$, define $Q_t^{n+1} = \bar{H}_t^* Q_t^n$. Then, the sequence $Q_t^n$ converges to the $t \textsuperscript{th}$ multi-timescale optimal value function as $n \rightarrow \infty$.
\label{lemma:1}
\end{lemma}

% \begin{proof}
% The full details of the proof are omitted due to space limitations, but are available in \cite{Emami-2023}.
% \end{proof}
\begin{proof}
Without loss of generality, it is sufficient to prove the case $t = 0$.

For any $Q_1, Q_2$ and $s_0 \in \mathcal{S}, \mathbf{a}_0 \in \mathbf{A}$,

\begin{align*}
&\phantom{ = }\lvert (\bH{0} Q_1)(s_0, \mathbf{a}_0) -  (\bH{0} Q_2 (s_0, \mathbf{a}_0) \rvert \\
&= \lvert (\T{0} \cdots \T{K-1} Q_1)(s_0, \mathbf{a}_0) - (\T{0} \cdots \T{K-1} Q_2)(s_0, \mathbf{a}_0) \rvert \\
&=  \Biggl{\lvert} \mathbb{E}_{s_1 \sim \Theta_0^C( p_0(s_0, \mathbf{a}_0))} \Bigl[ \Theta_0^C(r_0(s_0, \mathbf{a}_0)) \\
&\phantom{ = }+ \gamma \max_{\mathbf{a}_1}(\T{1} \cdots \T{K-1} Q_1)(s_1, \Gamma_{1,\mathbf{a}_0}^k(\mathbf{a}_1))\Bigr] \\
&\phantom{ = }- \mathbb{E}_{s_1 \sim \Theta_0^C( p_0(s_0, \mathbf{a}_0))} \Bigl[ \Theta_0^C(r_0(s_0, \mathbf{a}_0)) \\
&\phantom{ = }+ \gamma \max_{\mathbf{a}_1}(\T{1} \cdots \T{K-1} Q_2)(s_1, \Gamma_{1,\mathbf{a}_0}^k(\mathbf{a}_1))\Bigr] \Biggr{\rvert}\\
&= \gamma \Biggl{\lvert} \mathbb{E}_{\Theta_0^C(p_0)} \Bigl[ \max_{\mathbf{a}_1}(\T{1} \cdots \T{K-1} Q_1)(s_1, \Gamma_{1,\mathbf{a}_0}^k(\mathbf{a}_1))\\
&\phantom{ = }- \max_{\mathbf{a}_1}(\T{1} \cdots \T{K-1} Q_2)(s_1, \Gamma_{1,\mathbf{a}_0}^k(\mathbf{a}_1)) \Bigr] \Biggr{\rvert} \\
&\leq \gamma \Biggl{\lvert} \mathbb{E}_{\Theta_0^C(p_0)} \Bigl[ (\T{1} \cdots \T{K-1} Q_1)(s_1, \mathbf{a}_1^*)\\
&\phantom{ \leq }- (\T{1} \cdots \T{K-1} Q_2)(s_1, \mathbf{a}_1^*) \Bigr] \Biggr{\rvert}\\
&\phantom{ \leq }\text{ where } \mathbf{a}_1^* = \argmax_{\mathbf{a}} \Bigl[(\T{1} \cdots \T{K-1} Q_1)(s_1, \Gamma_{1,\mathbf{a}_0}^k(\mathbf{a}_1))\\
&\phantom{ \leq }- (\T{1} \cdots \T{K-1} Q_2)(s_1, \Gamma_{1,\mathbf{a}_0}^k(\mathbf{a}_1)) \Bigr]\\
&\leq \gamma \max_{s,\mathbf{a}} \Biggl{\lvert} (\T{1} \cdots \T{K-1} Q_1)(s, \mathbf{a})\\
&\phantom{ \leq }- (\T{1} \cdots \T{K-1} Q_2)(s, \mathbf{a}) \Biggr{\rvert}.
\end{align*}

We can continue to expand the inequality in a similar manner:
\begin{align*}
    &\forall s_0, \mathbf{a}_0,\\
    &\phantom{ \leq }\lvert (\bH{0} Q_1)(s_0, \mathbf{a}_0) - (\bH{0} Q_2)(s_0, \mathbf{a}_0) \rvert\\
    &\leq \gamma \max_{s,\mathbf{a}} \bigl{\lvert} (\T{1} \cdots \T{K-1} Q_1)(s, \mathbf{a}) - (\T{1} \cdots \T{K-1} Q_2)(s, \mathbf{a}) \bigr{\rvert}\\
    &\leq \gamma^2 \max_{s,\mathbf{a}} \bigl{\lvert} (\T{2} \cdots \T{K-1} Q_1)(s, \mathbf{a}) - (\T{2} \cdots \T{K-1} Q_2)(s, \mathbf{a}) \bigr{\rvert}\\
    &\phantom{\leq }\vdots\\
    &\leq \gamma^K \max_{s, \mathbf{a}} \bigl{\lvert} Q_1(s,\mathbf{a}) - Q_2(s, \mathbf{a}) \bigr{\rvert},
\end{align*}
which implies $||\bH{0}Q_1 - \bH{0}Q_2||_{\infty} \leq \gamma^K ||Q_1 - Q_2||_{\infty}$. 
Therefore $\bH{t}$ is a $\gamma^K$-contraction with respect to infinity norm, and by the Banach fixed-point theorem, $\bH{t} Q_t^* = Q_t^*$ is the unique fixed point solution for all $t$.
\end{proof}
It follows that the fixed points of $\bH{0},\dots,\bH{K-1}$ together make up the optimal multi-timescale value function, which is represented by the $K$ values $Q_0^{\pi^*},\dots,Q_{K-1}^{\pi^*}$.
Next, we show that these fixed points have the largest value compared to any other multi-timescale value function for any history-dependent policy $\bar{\pi} \in \Pi$.

\begin{lemma}
Let $\bH{t \bmod K} = \T{t \bmod K} \dots \T{(t + K - 1)\bmod K}$ be the $K$-step  multi-timescale Bellman optimality operator and $Q^{\pi^*}_{t \bmod K}$ be its fixed point. Then, for any history-dependent policy $\bar{\pi} \in \bar{\Pi}$, $Q^{\pi^*}_{t \bmod K}(s, \mathbf{a}) \geq Q_t^{\pi}(s, \mathbf{a})$.
\label{lemma:2}
\end{lemma}

% \begin{proof}
% The full details of the proof are omitted due to space limitations, but are available in \cite{Emami-2023}.
% \end{proof}
\begin{proof}
For any $\bar{\pi} \in \bar{\Pi}$, $t, s, \mathbf{a}$ and $Q$, the following inequality holds: 
\begin{align*}
    & \phantom{ \leq }(\bar{\mathcal{T}}_{t}^{\bar{\pi}} Q)(s_t, \mathbf{a}_t) \\
    &:= \Theta^C_t(r_t(s_t,\mathbf{a}_t))\\
    &\phantom{ \leq }+ \gamma \mathbb{E}_{s_{t+1} \sim \Theta^C_t(p_t), \mathbf{a}_{t+1} = \bar{\pi}} \bigl[ Q(s_{t+1}, \Gamma^k_{{t+1},\mathbf{a}_{t}}(\mathbf{a}_{t+1}) \bigr]\\
    &\leq \Theta^C_t(r_t(s_t,\mathbf{a}_t))\\
    &\phantom{ \leq }+ \gamma \max_{\mathbf{a}_{t+1}} \mathbb{E}_{s_{t+1} \sim \Theta^C_t(p_t)} \bigl[ Q(s_{t+1}, \Gamma^k_{t+1}(\mathbf{a}_{t+1}) \bigr]\\
    &= (\T{t \bmod K} Q)(s_t, \mathbf{a}_t).
\end{align*}
This implies
\begin{align*}
    &\phantom{ \leq  }(\bar{\mathcal{T}}_{t}^{\bar{\pi}} \bar{\mathcal{T}}_{t+1}^{\bar{\pi}} \dots \bar{\mathcal{T}}_{t + K-1}^{\bar{\pi}} Q)(s,\mathbf{a})\\
    &\leq (\T{t \bmod K} \T{(t+1) \bmod K} \dots \T{(t+K-1) \bmod K}Q)(s,\mathbf{a})\\
    &= (\bH{t \bmod K} Q)(s,\mathbf{a}).
\end{align*}
Therefore, 
\begin{align*}
&\phantom{ = }Q^{\pi}_t(s,\mathbf{a})\\
&= \lim_{n\rightarrow \infty} (\bar{\mathcal{T}}_{t}^{\bar{\pi}} \bar{\mathcal{T}}_{t+1}^{\bar{\pi}} \dots \bar{\mathcal{T}}_{t + Kn-1}^{\bar{\pi}} Q)(s,\mathbf{a})\\
&\leq \lim_{n \rightarrow \infty} ((\bH{t \bmod K})^n Q)(s,\mathbf{a}) = Q_t^{\pi^*}(s,\mathbf{a})
\end{align*}
holds, which concludes the proof.
\end{proof}

Finally, to prove the main claim in Theorem~\ref{thm:main}, by Lemma~\ref{lemma:2} it is sufficient to show $\lim_{n \rightarrow \infty}Q_t^{\pi^n} = Q_t^{\pi^*}$ for all $t \in \{0,\dots,K-1\}$.
To prove this, first, we must establish monotonic improvement of multi-timescale policies induced by $\bar{\pi}^n$ during policy iteration, i.e., $Q_t^{\pi^{n+1}} (s, \mathbf{a}) \geq Q_t^{\pi^{n}}(s, \mathbf{a})$ always holds for all $t, s, \mathbf{a}, n$. This result follows by proving a similar result to Theorem 2 in Lee et al.~\cite{lee2020reinforcement} in the same manner as the proofs for Lemmas~\ref{lemma:1} and~\ref{lemma:2}, i.e., by projecting the reward and transition functions with $\Theta^C_t$.
We refer to~\cite{lee2020reinforcement} for the details. 
When the policy is no longer improving, i.e., $\pi^{n+1} = \pi^n$ and $Q_t^{\pi^{n+1}} = Q_t^{\pi^n}$, then by definition, $Q_t^{\pi^n}$ satisfies the multi-timescale Bellman optimality equation.
By Lemma~\ref{lemma:1}, the multi-timescale Bellman optimality equation has a unique solution $Q_t^{*}$, implying $Q_t^{*} = Q_t^{\pi^n}$. Thus, $\pi^n$ is the optimal multi-timescale policy.
Moreover, at every step of policy iteration, we have a policy $\bar{\pi}^n$ which is in $\bar{\Pi}$ and which induces a multi-timescale policy $\pi^n$ that is guaranteed to eventually converge to the optimal multi-timescale policy.
Therefore, the optimal multi-timescale policy always exists.

The known inequality $Q^*_{\text{Dec-POMDP}} \leq Q^*_{\text{ MDP }}$ establishes that the optimal value function of a  multi-timescale multi-agent MDP is an upper bound of the optimal value function for \texttt{MT-DEC-POMDP}. 
This implies that the optimal value function obtained by policy iteration under Assumption 1 may overestimate the true optimal value function.
See Theorem 5.1 in Oliehoek et al.~\cite{Oliehoek2008} for proof details. 

\section{Phasic Policy Gradient Method}
\label{sec:tdc}

\revision{The theory from the previous section suggests encoding $K$-periodicity into learning algorithms for \texttt{MT-DEC-POMDP} to encourage learning the optimal policy when $k$ and $C$ are known. Let the integer $\triangle_t \in \{0,\dots,K-1\}$ indicate the \revision{current \emph{phase}, i.e., $t \bmod K$}. A straightforward approach} is to encode this phase as a one-hot vector of size $K$, $\texttt{O.H.}(\triangle)$, which can be concatenated to each agent's observation \revision{$[o^i_t ; \texttt{O.H.}(\triangle_t)]$}.
However, when the mapping between the phase-augmented observation and the optimal action is complex, encoding the phase as a one-hot vector may not be sufficient to provide a good inductive bias for \revision{$K$-periodicity.}

Alternatively, we can parameterize each agent with phase-functioned neural networks~\cite{Holden2017} (PFNNs).
PFNNs are spline-based neural architectures whose weights smoothly vary as a function of the current phase.
%A spline is used to determine the weights of each layer that smoothly interpolates between a fixed set of learned weights.
This provides an inductive bias of using similar weights for adjacent phases and reusing the same network weights at time steps separated by a specified period.
\textbf{\revision{Favorably, t}he number of parameters in PFNNs scales proportionally with the number of spline control points (a constant) and \emph{not} with the period $K$}. 
The use of PFNNs in RL is under-explored, with only one known previous use for training single agents in cyclic environments~\cite{Sharma}.
\revision{Each layer $l$ of a PFNN has} a weight matrix $\alpha$ computed by a \emph{phase function} $\alpha_l = \Theta(\beta_l; \revision{ 2\pi \triangle_t/K)}$ conditioned on learnable weight matrices $\beta_l$ and phase \revision{ $2\pi \triangle_t/K \in [0, 2\pi]$}.
Following~\cite{Holden2017}, we use a Catmull-Rom spline for $\Theta$, which is a cubic spline with \textbf{four} learnable spline control points $\beta_l = [\beta_l^0,\beta_l^1,\beta_l^2,\beta_l^3]$.
The weight for layer $l$ is 
\begin{align}
    \alpha_l &= \beta_l^{x_1} + w (\frac{1}{2} \beta_l^{x_2} - \frac{1}{2} \beta_l^{x_0}) \nonumber\\
    &+ w^2 ( \beta_l^{x_0} - \frac{5}{2} \beta_l^{x_1} + 2 \beta_l^{x_2} - \frac{1}{2} \beta_l^{x_3}) \nonumber\\
    &+ w^3 (\frac{3}{2} \beta_l^{x_1} - \frac{3}{2} \beta_l^{x_2} + \frac{1}{2} \beta_l^{x_3} - \frac{1}{2} \beta_l^{x_0}), \nonumber
\end{align}
where \revision{$w = 4 \triangle_t / K \hspace{.5em}(\texttt{mod } 1)$} and \revision{$x_n = \lfloor 4 \triangle_t / K \rfloor + n - 1\hspace{.5em}(\texttt{mod } 4)$}.
The bias for layer $l$ is computed in a similar fashion.
The start and end control points for each layer are the same, making each PFNN layer cyclic. In this work, we adapt the actor-critic policy gradients method COMA~\cite{Foerster2017} by using PFNNs with period $K = \text{LCM}(\Tilde{K},C)$ for the actor and critic networks.
% %\input{sections/algorithm}
\section{Experiments}
\label{sec:experiments}
\subsection{The Move Box Problem}
\label{sec:movebox}
\textbf{Setup: }We adapted a gridworld environment called Move Box~\cite{jiang2021multi} to create a toy multi-timescale environment with a time-dependent optimal policy.
That is, the two agents need to coordinate their actions to push a green box to a goal location (``G'') within a maximum of 20 steps.
In the \textbf{easy} version, one agent (red) is a ``fast'' agent that uses $k_1 = 1$ and one agent (blue) is a ``slow'' agent that acts every two steps $k_2 = 2$. 
The period provided to time-aware agents is therefore $K := \texttt{LCM}(1, 2, C = 1) = 2$.
For the \textbf{hard} version of this task, the fast agent uses $k_1 = 2$ and slow agent uses $k_2 = 3$, thus $K := \texttt{LCM}(2, 3, C = 1) = 6$.
Each agent receives a 4D partial observation consisting of its own position and the position of the box.
To avoid the need for a complex exploration strategy, we restrict the action space to 3 discrete actions: move up, move down, or null (do nothing).
To push the green box up or down, the agents have to be on either side of the box and move in the same direction at the same time.
Agents receive a reward of +1 whenever the box moves towards the goal and a reward of +20 once the box reaches the goal.
The easy version can be solved with 7 actions while the hard version requires 19 actions; there is a significant increase in exploration difficulty between the easy and hard versions.
To implement multiple timescales, the only legal action available to the slow agent between steps is the null action.
\emph{Move Box is designed so that a time-unaware fast agent suffers from observation aliasing} (Sec.~\ref{sec:bg}).
A time-unaware fast agent's limited information means it cannot precisely determine whether to move up or do nothing to synchronize its actions with the slow agent.

\textbf{Agents: }We train four COMA-based agents:
\begin{itemize}
    \item \texttt{Basic COMA}, a time-unaware agent that uses a feedforward neural network for the actor and critic networks.
    \item \texttt{Recurrent COMA}, a time-aware agent that uses an RNN for the actor and critic networks to condition on the full history up to the current time step\revision{~\cite{wang2020r}}.
    \item \texttt{One-Hot (O.H.) phase-aware COMA}, an agent whose observations are augmented with a one-hot encoding of the current phase $\triangle_t$.
    \item \texttt{Phasic COMA}, an agent whose actor and critic networks are PFNNs with weights indexed by $2 \pi \triangle_t/K$.
\end{itemize}
\revision{We also run a variant of PFNN with period 4 instead of $K$, \texttt{Phasic COMA (4)}, to explore performance sensitivity to this hyperparameter.}
In the easy environment $4 > K$ and in the hard environment $4 < K$.
All actor and critic networks share parameters.
We take the standard approach of providing a one-hot agent ID as an auxiliary input to distinguish between agents.
\begin{figure}[t]
    \centering
    \includegraphics[width=.77\columnwidth]{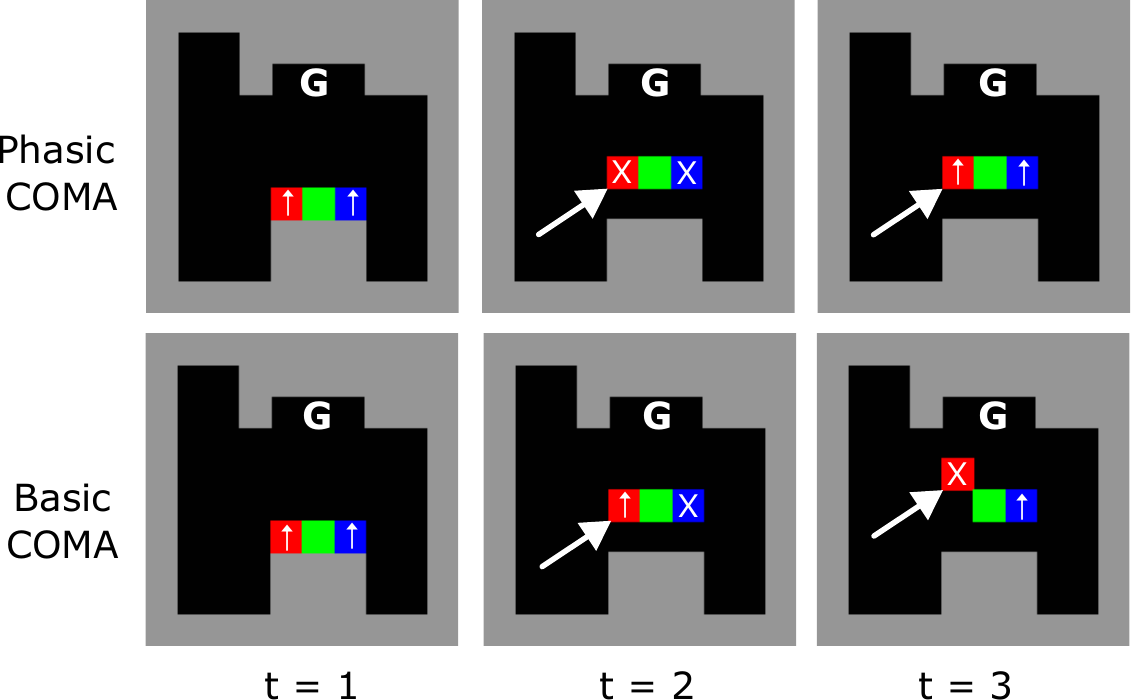}
    \caption{\textbf{Move Box qualitative analysis}. In the easy setting, the red agent acts every step ($k_1 = 1$) and the blue agent acts every two steps ($k_2 = 2$). The joint action is shown in white on each agent. The time-unaware \texttt{Basic COMA} red agent tries to move the box up at $t = 2$, which causes it to drop the box at $t = 3$. The time-aware \texttt{Phasic COMA} red agent learns to take a null action at $t = 2$. The hard setting requires more sophisticated coordination between agents. Best viewed in color.}
    \label{fig:moveboxes_demo}
\end{figure}
\begin{figure}[t]
    \centering
    \includegraphics[width=.82\columnwidth]{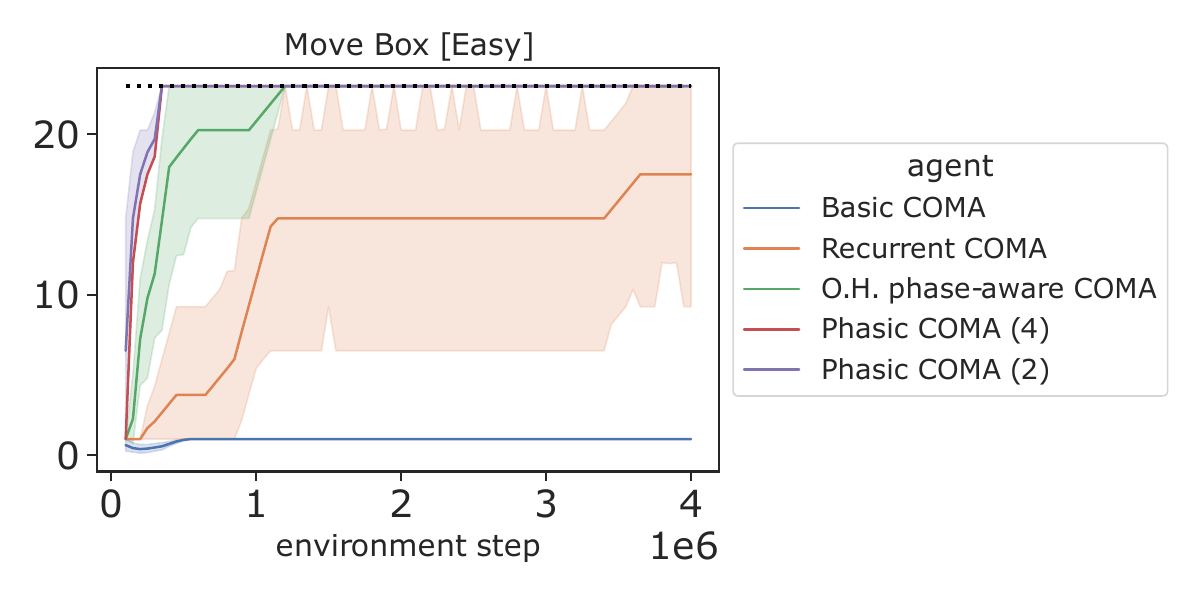}
    \includegraphics[width=.82\columnwidth]{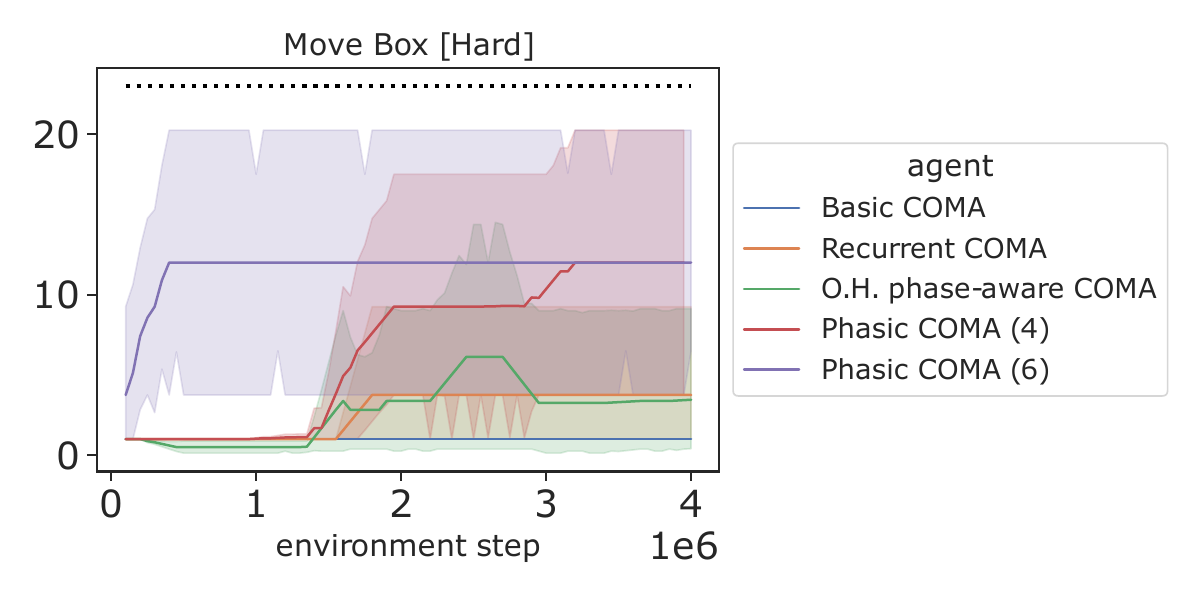}
    %\end{subfigure}
    \caption{\revision{\textbf{Move Box return vs. steps}. Mean return over 8 random seeds of the learned greedy policy at various steps during training (shaded region is the 95\% confidence interval (CI)). Best possible return shown by black dotted line. The PFNN-based agent \texttt{Phasic COMA} with correct periods $K = 2,6$ achieves the highest reward in the fewest environment steps.\label{fig:moveboxes_plot}}}
\end{figure}
\begin{table}[t]
\footnotesize
\caption{\textbf{Move Box results.} Success (\%) is the \revision{fraction out} of 8 random seeds that the agents learned to move the box to the goal. The optimal return is 23.0.}
\label{tab:movebox}
\begin{adjustbox}{max width=\columnwidth}
\begin{tabular}{@{}lcccc@{}}
\toprule
                      & \multicolumn{2}{c}{Move Box [Easy]} & \multicolumn{2}{c}{Move Box [Hard]} \\ \midrule
                     COMA & Success (\%)    & Avg. return    & Success (\%)    & Avg. return    \\ \midrule
Basic            & 0               & 1.0\tiny{$\pm 0.0$}               & 0               & 1.0\tiny{$\pm 0.0$}               \\
Recurrent        & 75              & 17.5\tiny{$\pm 10.2$}       & 0               & 3.75\tiny{$\pm 7.8$}        \\
O.H. phase-aware & \textbf{100}             & \textbf{23.0}\tiny{$\pm 0.0$}              & 0               & 3.45\tiny{$\pm 7.9$}        \\
\revision{Phasic (4)  }        & \textbf{100}             & \textbf{23.0}\tiny{$\pm 0.0$}              & \textbf{50}              & 12.0\tiny{$\pm 11.8$}       \\
Phasic (15)          & \textbf{100}             & \textbf{23.0}\tiny{$\pm 0.0$}              & \textbf{50}              & 12.0\tiny{$\pm 11.8$}       \\ \bottomrule
\end{tabular}
\end{adjustbox}
\end{table}

\textbf{Results: }Table~\ref{tab:movebox} shows quantitative results and Fig.~\ref{fig:moveboxes_plot} compares \revision{test return as a function of environment steps}. 
Both \texttt{O.H. phase-aware COMA} and \texttt{Phasic COMA} learn to reliably solve the easy Move Box environment across all random seeds, with \texttt{Phasic COMA} demonstrating a small advantage in terms of efficiency.
\texttt{Recurrent COMA} \revision{needs more steps to achieve a good test return yet ultimately} performs less reliably.
\texttt{Basic COMA} fails on this environment as expected due to observation aliasing (Fig.~\ref{fig:moveboxes_demo}).
In the hard version (Fig.~\ref{fig:moveboxes_plot}), only \texttt{Phasic COMA} learns to solve the environment on \revision{just} 50\% of the training runs.
\revision{The PFNN-based actor and critic is robust to a slightly smaller or larger period than $K$, although it appears to require more environment steps.}

\subsection{Building Energy Management}
\begin{figure}[t]
    \centering
    \includegraphics[width=.82\columnwidth]{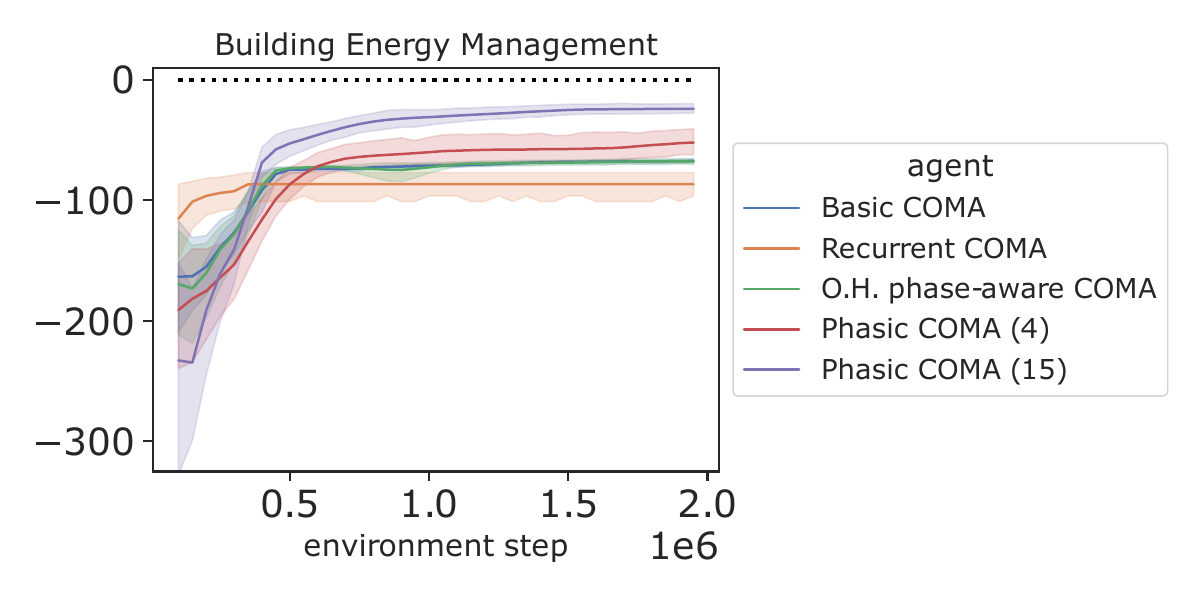}
    \caption{\textbf{BEM \revision{results}.} Mean return over 8 random seeds of the learned greedy policy at various steps during training (shaded region is the 95\% CI). Best possible is the black dotted line.  The \texttt{Phasic COMA} agent with correct period $K = 15$ outperforms \revision{non-phasic} variants by a wide margin.}
    \label{fig:bem_plot}
    \vspace{-1.em}
\end{figure}
\begin{figure}[t]
    \centering
    %\begin{subfigure}{\columnwidth}
        \centering
        \includegraphics[scale=.4]{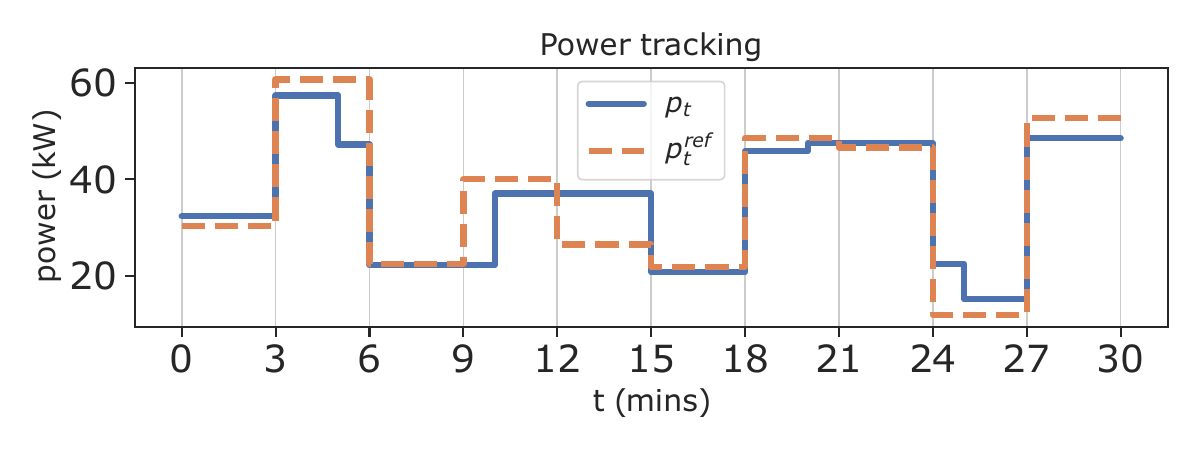}
        \caption{\textbf{BEM qualitative results}. Total power vs. power reference signal from one of the \texttt{Phasic COMA} training runs.\label{fig:qual_pt}}
\end{figure}
\begin{figure}[t]
        \centering
        \includegraphics[scale=.4]{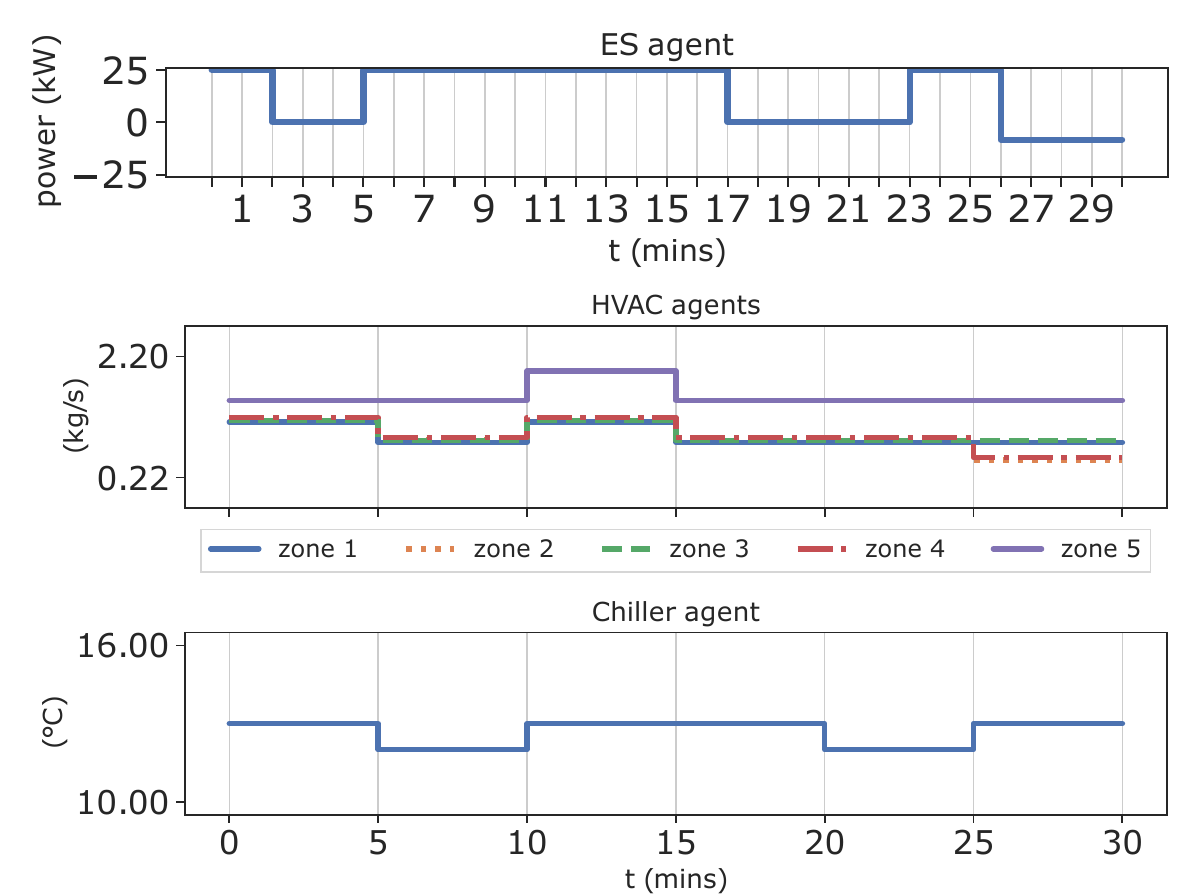}
        \caption{\textbf{BEM qualitative results}. Agent actions from one of the \texttt{Phasic COMA} training runs.\label{fig:building-actions}}
    \vspace{-2em}
\end{figure}
\textbf{Setup: }In this environment, agents attempt to coordinate the control of HVAC and energy storage (ES) for a five-zone small office building. 
See Biagioni et al.~\cite{biagioni2022model} for details about the reduced order model used to simulate the building.
There are 7 agents: an HVAC agent per zone able to change the mass flow rate (kg/s) every 5 mins, an HVAC chiller agent that can change the discharge air temperature (\degree{C}) every 5 mins, and an ES agent that can change its charging or discharging power (kW) every 1 minute.
The goal is for the agents to coordinate their total power consumption to track a reference power signal that changes every 3 mins while minimizing discomfort to building occupants.
The control horizon is set to 30 mins ($dt = 1$min).
The global reward function at time step $t$ is defined as
\begin{equation*}
    r_t = - \sum_{\text{zone}_i} \big( (T_t^i - \overline{T})^{+} + (\underline{T} - T_t^i)^{+} \big)^2 - \alpha (p_t - p_t^{\text{ref}})^2,
\end{equation*}
where $\alpha = 0.01$, $\overline{T} := 26$ \degree{C} and $\underline{T} := 24$ \degree{C} is the thermal comfort band, $T_t^i$ is zone $i$'s temperature, and $p_t$ is the total power.
To implement multiple timescales, agents repeat their previous action between steps. 
The period used for learning periodic non-stationary policies is $K := \texttt{LCM}(1,5, C = 3) = 15$, where $C$ encodes the cyclic power reference signal.

\textbf{Results: }Out of all agents, only $\texttt{Phasic COMA}$ is able to reliably learn a near-optimal joint policy (Fig.~\ref{fig:bem_plot}).
\revision{The variant with arbitrary PFNN period $4 \ll K$ is the second best agent}. 
We visualize the control actions selected by the joint policy from one of the \texttt{Phasic COMA} runs in Fig.~\ref{fig:building-actions}.
The slow HVAC agents have successfully learned to coordinate with the faster ES agent to track the reference signal (Fig.~\ref{fig:qual_pt}) without violating the thermal comfort band; for example, by increasing their power consumption between 10-15 mins while the ES agent is already maximally discharging.

\section{CONCLUSIONS}
In this work, we proposed a multi-timescale MARL framework for learning policies that can represent complex time-dependent behaviors.
We introduced the multi-timescale non-stationary joint policy as the policy induced by a $K$-periodic non-stationary joint policy, where the period $K$ is given by knowledge about agent timescales and cyclic environment components, both of which are typically known \emph{a priori}.
We use phase-functioned neural networks to introduce an inductive bias for learning a periodic non-stationary joint policy.
Our results on grid world and building energy management environments establish the effectiveness of our framework, suggesting that follow-up work could explore using it to solve more advanced power systems problems.

\bibliographystyle{IEEEtran}
\bibliography{main}

\end{document}